\documentclass[11pt]{article}

\usepackage{lineno}
\usepackage{graphicx, hyperref}

\modulolinenumbers[5]
\usepackage{color}
\usepackage{braket}
\usepackage{amsmath,amssymb}
\usepackage{mathrsfs}
\usepackage[top=30truemm,bottom=30truemm,left=40truemm,right=40truemm]{geometry}
\usepackage{indentfirst}
\usepackage{url}

\usepackage{wrapfig}
\usepackage{subcaption}

\usepackage{bm}
\usepackage{amssymb}
\usepackage{amsbsy}
\usepackage{amscd}
\usepackage{amsmath}
\usepackage{amsthm}
\usepackage{color}

\newtheorem{theorem}{Theorem}[section]
\newtheorem{lemma}[theorem]{Lemma}

\newcommand{\RR}{\mathbb R}

\newcommand{\EE}{\mathbb E}

\begin{document}

\title{Bayesian Free Energy of Deep ReLU Neural Network in Overparametrized Cases}

\author{Shuya Nagayasu and Sumio Watanabe\\
Department of Mathematical and Computing Science\\
Tokyo Institute of Technology,\\
Mail-Box W8-42, 2-12-1, Oookayama, \\
Meguro-ku, Tokyo,
152-8552, Japan
}

\date{}

\maketitle

\begin{abstract}
In many research fields in artificial intelligence, 
it has been shown that deep neural networks
are useful to estimate unknown functions on high dimensional input spaces. 
However, their generalization performance 
is not yet completely clarified from the theoretical point of view 
because they are nonidentifiable and singular learning machines. 
Moreover, a ReLU function is not differentiable, to which algebraic or analytic 
methods in singular learning theory cannot be applied. 

In this paper, we study a deep ReLU neural network in overparametrized cases 
and prove that the Bayesian free energy, which is equal to the minus log marginal likelihood
or the Bayesian stochastic complexity, 
is bounded even if the number of layers are larger than necessary to 
estimate an unknown data-generating function. Since the Bayesian generalization error is equal to the 
increase of the free energy as a function of a sample size, our result also shows that
the Bayesian generalization error does not increase even if a deep ReLU neural
network is designed to be sufficiently large or in an opeverparametrized state. 
\end{abstract}

\section{Introduction}

Deep neural networks are now being used in many fields,
for example, pattern recognition, robotic control, bioinformatics, data science, 
time series prediction, and so on. Their high performances have been shown in many 
experiments, however mathematical foundation to study them is not yet completely 
established. 

It is well known that the generalization of machine learning is basically divided into two elements, Bias and Variance. 
The bias and the variance are determined by approximation ability and  the complexity of the model respectively. 
If a machine leaning model has larger number of parameter, the bias gets smaller and the variance gets larger in general. 
It is known as Bias-Variance Tradeoff.
Nevertheless deep neural network in practical use has huge number of parameters, they are well generalized and the influence of variance is smaller comparing to such a number of parameters. The reason of this phenomena is not clear, because deep neural networks are nonidentifiable and singular\cite{Watanabe2001b} \cite{Watanabe2007}.

A learning machine is called identifiable if a map from a parameter to a probability 
distribution is one-to-one, and regular if it has a positive-definite Fisher information 
matrix. If a learning machine is identifiable and regular, then the
regular statistical theory holds \cite{Amari1993}, resulting that asymptotic normality of
the maximum likelihood estimator holds and that 
AIC \cite{Akaike1974} and BIC \cite{Schwarz1978} can be employed
in model selection problems. However, if 
learning machines contain hierarchical structure or hidden variables, 
they are nonidentifiable and singular, for example, 
layered neural networks \cite{Watanabe2001a,Aoyagi2012}, 
normal mixtures \cite{Hartigan1985,Yamazaki2003}, 
matrix factorizations \cite{Nakajima2010}, reduced rank regressions \cite{Aoyagi2005}, 
Poisson mixtures \cite{Sato2019}, latent Dirichlet allocation \cite{Hayashi2021}, 
and Boltzmann machine \cite{Yamazaki2005,Aoyagi2010}.
In these learning machines, it was shown that 
singularities make the Bayesian generalization errors smaller than that of regular 
models, even in under-parametrized cases \cite{Nagayasu2022}.
These phenomena are called implicit or intrinsic regularization
\cite{Nakajima2010,Wei2022,Tokuda2022}, and researches about
quantitative effects caused by singularities have been applied to model selection 
\cite{Drton2017}, hyperparameter design \cite{Yamazaki2013}, and optimization of 
Markov chain Monte Carlo method for the posterior distribution \cite{Nagata2008}. 

Statistical properties of nonidentifiable and singular learning machines 
are now being clarified by these researches, however, the conventional singular 
learning theory is based on the condition that a log likelihood is
an algebraic or analytic function of a parameter, hence it cannot be applied to 
non-differentiable ReLU neural networks. In this paper, we study 
singular learning theory so as to be employed in the non-differentiable cases and derive 
the upper bound of the free energy of deep ReLU functions
 in overparametrized cases. 

The Bayesian free energy $F_n$ is mathematically equal to the minus log marginal likelihood
and Bayesian stochastic complexity, where $n$ is a sample size. 
The free energy plays an important role in Bayesian learning theory, because 
the average generalization error $\EE[G_n]$ 
satisfies a formula  for an arbitrary $n$ \cite{Watanabe2009},
\[
\EE[G_n]=\EE[F_{n+1}]-\EE[F_n]-S, 
\]
where $G_n$ is the Kullback-Leibler divergence from a data-generating 
distribution to the Bayesian predictive distribution and 
$S$ is the entropy of the data-generating distribution. 
If a log likelihood function is analytic or algebraic, then 
by the conventional singular learning theory, 
it was proved that, if a data-generating distribution is realizable by 
a learning machine, 
\[
\EE[F_n]=nS+\lambda \log n + O(n\log n),
\]
where $\lambda>0$ is the real log canonical threshold (RLCT), 
resulting that 
\[
\EE[G_n]=\frac{\lambda}{n}+o(1/n).
\]
In this paper, we prove that the Bayesian free energy of deep 
ReLU neural network satisfies the following inequality
\[
\EE[F_n]\leq nS + \lambda_{ReLU} \log n + C
\]
where the constant $\lambda_{ReLU}>0$ is bounded even if the number of 
layers in a learning machine is larger than necessary to estimate 
the data-generating distribution. 
Hence, if the generalization error has asymptotic expansion, 
then it should be 
\[
\EE[G_n] \leq \frac{\lambda_{ReLU}}{n}+o(1/n). 
\]
In practical applications, we 
do not know the appropriate number of layers of a deep network for unknown 
data-generating distribution, hence a sufficiently large neural network
is often employed. Our result shows that, 
even if a deep ReLU neural network is designed 
in an over-parametrized state, the Bayesian generalization error 
is bounded by a constant defined by the data-generating distribution. 

This paper consists of seven sections. In the second section, we prepare 
mathematical framework of Bayesian learning theory.
In the third section, a deep ReLU neural network is explained and
 the main theorem is introduced. In the fourth and fifth sections, several lemmas 
 and the main theorem are proved respectively. 
In the sixth and seventh sections, we discuss the main 
result and conclude this paper. 

\section{Framework of Bayesian Learning theory}

In this section, we briefly explain the framework of Bayesian Learning for supervised learning used in this paper. 

Let $X^n = (X_1,\cdots X_n) and Y^n = (X_1,\cdots X_n)$ be samples independently and identically taken from the true probability distribution $q(x,y) = q(y|x)q(x)$. Also let $p(y|x,\theta), \varphi(\theta)$ be the statistical model and prior distribution. The posterior distribution is given by

\begin{align}
p(\theta|X^n,Y_n) = \frac{1}{Z(Y^n|X^n)}\varphi(\theta)\prod_{i=1}^{n}p(Y_i|X_i, \theta).
\end{align}
where $Z_n$ is normalizing constant denoted as marginal likelihood: 
\begin{align}
Z_n = \int \varphi(\theta)\prod_{i=1}^{n}p(Y_i|X_i,\theta)\mathrm{d}\theta.
\end{align}
Hence $Z_n$ is the probability distribution for $Y^n$ conditioned by $X^n$ estimated from samples, $Z_n$ is denoted by $p(Y^n|X^n)$ as probability distribution. The posterior predictive distribution is defined by the average of statistical model over posterior distribution:
1\begin{align}
p(y|x,X^n,Y_n) = \int p(y|x,\theta)p(\theta|X^n,Y^n)\mathrm{d}\theta.  
\end{align}
In Bayesian Learning the true distribution $q(x)$ is inferred by predictive distribution $p(x|X^n)$. For comparing $q(x)$ with $p(x|X^n)$, the free energy$F_n$ and the generalization loss $G_n$ are used.

The free energy is negative log value of marginal likelihood
\begin{align}
F_n = -\log Z_n.
\end{align}
The free energy is also called evidence, stochastic complexity and used for model selection\cite{Schwarz1978}\cite{Rissanen_1978}\cite{S_Watanabe2013}\cite{Drton2017}, We introduce the following variables to explain why the free energy is used for comparing true and predictive distributions.
\begin{align}
S &= -\int q(y|x)q(x) \log q(y|x) \mathrm{d}x\mathrm{d}y, \\
S_n &= \frac{1}{n}\sum_{i = 1}^{n} \log q(Y_i|X_i).
\end{align}
The entropy$S$ is the average negative log likelihood of the true distribution and empirical entropy$S_n$ is average of log loss. From definition, the following equation holds
\begin{align}\label{eq:111}
F_n = nS_n + \log \frac{q(Y^n|X^n)}{p(Y^n|X^n)}
\end{align}
where $q(Y^n|X^n) = \prod_{i = 1}^{n}q(Y_i|X_i)$. The average of $F_n$ over the sample generating $(X^n,Y^n)$ is
\begin{align}
\EE[F_n] &= nS + \int q(y^n|x^n)q(x^n) \log \frac{q(y^n|x^n)q(x^n)}{p(y^n|x^n)q(x^n)}\mathrm{d}x^n\mathrm{d}y^n \\
&= nS + D_{\mathrm{KL}}(q(x^n,y^n)|p(y^n|x^n)q(x^n))
\end{align}
The entropy is not depends on statistical model and prior distribution, therefore the expected value of free energy over sample generating  is equivalent to the Kullback Leibler divergence between true distribution for sample size and evidence except for the constant.

The generalization error is the Kullback Leibler divergence between true distribution and predictive distribution.
\begin{align}
G_n = D_{\mathrm{KL}}(q(y|x)q(x)|p(y|x,X^n,Y^n)q(x))
\end{align}
The generalization error is also used for model selection\cite{Akaike1974}\cite{S_Watanabe2010a}. 
For $p(y|x,X^n,Y^n) = p(Y_{n+1}|X_{n+1},X^n,Y_n)$, the following equation holds
\begin{align}
p(y|x,X^n,Y^n) &= \frac{1}{Z_n}\int p(Y_{n+1}|X_{n+1},\theta)\varphi(\theta)\prod_{i=1}^{n}p(Y_i|X_i,\theta)\mathrm{d}\theta\\
&= \frac{Z_{n + 1}}{Z_n}. \label{eq:newdata}
\end{align} The average of a sample of negative log of equation \eqref{eq:newdata} is following
\begin{align}
\EE[G_n] - S = \EE[F_{n+1}] -\EE[F_{n}].
\end{align}
This equation show that the average generalization error is equal to the increase of the average 
free energy. If the log likelihood is an algebraic or analytic function, 
by using algebraic geometric foundation \cite{Hironaka1964}\cite{Atiyah1970},
it was proved that asymptotic behaviors of $\EE[F_n]$ and $\EE[G_n]$ are 
given by the real log canonical threshold \cite{Watanabe2009}. 
However, since a ReLU function is neither algebraic nor analytic, conventional 
singular learning theory cannot be employed. 

Table \ref{table:average} show the definitions and notations used in this paper.

\begin{table}[h]
 \caption{Notation}
 \label{table:average}
 \centering
  \begin{tabular}{cll}
   \hline
   Notation & Definition & Name \\
   \hline \hline
   $\EE[\cdots]$ & $ \int \cdots \prod_{i=1}^{n}q(Y_i|X_i)q(X_i)\mathrm{d}X^n\mathrm{d}Y^n$ & average of generating of samples  \\
   $\EE_\theta[\cdots]$ & $ \int \cdots p(\theta|X^n)\text{d}w $ & average of posterior \\
   $\EE_{x,y}[\cdots]$ & $ \int \cdots q(y|x)q(x)\text{d}x $ & average of true distribution \\
   $S$ & $ -\EE_{x,y}[\log q(y|x)q(x)]$ & entropy \\
   $S_n$ & $-\frac{1}{n}\sum_{i=1}^{n}\log q(X_i)$ & the empirical entropy \\
   $F_n$ & $-\log Z_n$ & the free energy \\
   $G_n$ & $\EE_{x,y}[\log q(y|x)/p(y|x,X^n,Y^n)]$ & the generalization error \\
   \hline
  \end{tabular}
\end{table}

\section{Deep Neural Network}
\subsection{Deep Neural Network as statistical model}
 In this section we describe the function of $N$-layer neural network. We define $H_{k} (1 \leq k \leq N) $ as the width of each layer.  Let $x \in \RR^{H_{1}}$ be input vector. $\sigma(t) = \{\sigma_i(t)\}$ is vector of activation function. The function from $k-1$-th layer to $k $-th layer $f^{(k)} \in \RR^{H_{k}}$ is defined by
\begin{align}
&f^{(1)}(w,b,x) = x \\
&f^{(k)}(w,b,x) = \sigma(w^{(k)}f^{(k-1)}(w,b,x) + b^{(k)}) \;\;\; (2 \leq i \leq N - 1)
\end{align}
where $w^{(k)} \in \RR^H_{k+1} \times R^H_{k}$ is weight matrix, $b^{(k)} \in \RR^{H_{k}}$ is bias. We collectively denote weight and bias as $w, b$ 
\begin{align}
w = (w^{(2)}, \cdots, w^{(N)}),
b = (b^{(2)}, \cdots, b^{(N)}),
\end{align}
There exists various activation functions for neural network such that ReLU, Sigmoid, Swish and so on. In this paper we analyze the case using ReLU active function which is defined by

\begin{align}
\sigma_i(t) = 
\begin{cases}
t_i & (t_i \leq 0) \\
0 &  (t_i < 0)
\end{cases}
\end{align}

For using Neural Network in Bayesian Learning, the relationship between input vector and output vector is stochastically modeled with the function $f^{(N)}(w,b,x)$. Let $y \in R^{H_{N}}$ be output vector. This paper concern the following statistical model:
\begin{align}
y = f^{(N)}(w,b,x) + N(0,I_{H_{N}}),
\end{align}
where $N(0,I_{H_{N}})$ is $H_{N}$ dimensional Gaussian noise which covariance is the identity matrix.  This model is represented as probability density function as follows:
\begin{align}
p(y|w,b,x) =  \frac{1}{\sqrt{2 \pi }^{H_{N}}}-\exp \frac{1}{2}||y - f^{(N)}(w,b,x)||^2.
\end{align}

\subsection{How the model realize the true}
In data analysis using neural network, the case that the model is larger than the data generating process is common. Such situation is called overparametrize. For analyzing such a overparameterized situation, we assume that the statistical model include the data generating process. 

We assume that the true probability distribution is $N^{*}$-layer ReLU neural network which has  $H^*_{k}$ width with parameters$(w^*,b^*)$. In this situation, the true distribution is 
\begin{align}
q(y|x) =  \frac{1}{\sqrt{2 \pi }^{H^*_{N}}}-\exp \frac{1}{2}||y - f^{(N^{*})}(w^*,b^*,x)||^2.
\end{align}
We show that if statistical model is a $N$-layer ReLU neural network that satisfies
\begin{align}\label{eq:cond1}
N^* \leq N,\;\;\;\;
H_1^*=H_1, \;\;\;\; H_{N^*}^*=H_N,
\end{align}
and
\begin{align}
H_k^* & \leq H_k\;\;\;(2\leq k\leq N^*-1), \label{eq:cond2}
\\
H_{N^*-1}^*& \leq H_{k}\;\;(N^*\leq k\leq N-1).\label{eq:cond3}
\end{align}
there exists a parameter $\hat{w}, \hat{b}$ which satisfies 
\begin{align}
q(y|x) = p(y|\hat{w}, \hat{b},x).
\end{align}
Such parameters are called optimal parameter. 

To describe optimal parameter, we define the following notation.
{\bf Notation}. 
Note that the dimension of the vector $f^{(k)}(w,b,x)$
 is different from $f^{(k)}(w^*,b^*,x)$ in general. 
A vector $v^{(k)}\in \RR^{H_k}$, which has the same dimension as the output vector of the $k$th layer
of the learning machine, is represented by 
\[
v^{(k)}=\left(
\begin{array}{c}
v_A^{(k)}
\\
v_B^{(k)}
\end{array}\right),
\]
where, for $1\leq k\leq N^*-1$, 
\[
v_A^{(k)}\in \RR^{H^*_k},\;\;\;v_B^{(k)}\in \RR^{H_k-H^*_k},
\]
or, for $N^*\leq k\leq N-1$, 
\[
v_A^{(k)}\in \RR^{H^*_{N^*-1}},\;\;\;v_B^{(k)}\in \RR^{H_k-H^*_{N^*-1}}.
\]
For example, the output of the $k$th layer is represented by 
\[
f^{(k)}(w,b,x)=\left(
\begin{array}{c}
f_A^{(k)}(w,b,x)
\\
f_B^{(k)}(w,b,x)
\end{array}\right),
\]
and the bias of the $k$th layer is represented by 
\[
b^{(k)}=\left(
\begin{array}{c}
b_A^{(k)}
\\
b_B^{(k)}
\end{array}\right). 
\]
A matrix $L^{(k)}\in \RR^{H_k\times H_{k-1}} $ whose size is equal to
the weight parameter from the $(k-1)$th layer to the $k$th layer is represented by 
\[
L^{(k)}=\left(
\begin{array}{cc}
L_{AA}^{(k)} & L_{AB}^{(k)} 
\\
L_{BA}^{(k)} & L_{AB}^{(k)} 
\end{array}\right),
\]
where, for $2\leq k\leq N^*-1$, 
\begin{eqnarray*}
&L_{AA}^{(k)}\in \RR^{H^*_k\times H^*_{k-1}}, &
L_{AB}^{(k)}\in \RR^{H^*_k\times (H_{k-1}- H^*_{k-1})}, 
\\
&L_{BA}^{(k)}\in \RR^{(H_k-H^*_k)\times H^*_{k-1}}, &
L_{BB}^{(k)}\in \RR^{(H_k-H^*_k)\times (H_{k-1}- H^*_{k-1})}, 
\end{eqnarray*}
for $N^*\leq k\leq N-1$, 
\begin{eqnarray*}
&L_{AA}^{(k)}\in \RR^{H^*_{N^*-1}\times H^*_{N^*-1}}, &
L_{AB}^{(k)}\in \RR^{H^*_{N^*-1}\times (H_{k-1}- H^*_{N^*-1})}, 
\\
&L_{BA}^{(k)}\in \RR^{(H_k-H^*_{N^*-1})\times H^*_{N^*-1}}, &
L_{BB}^{(k)}\in \RR^{(H_k-H^*_{N^*-1})\times (H_{k-1}- H^*_{N^*-1})}, 
\end{eqnarray*}
or, for $k=N$,
\begin{eqnarray*}
&L_{AA}^{(N)}\in \RR^{H^*_{N^*}\times H^*_{N^*-1}}, &
L_{AB}^{(N)}\in \RR^{H^*_{N^*}\times (H_{N-1}- H^*_{N^*-1})}.
\end{eqnarray*}
For example, a weight parameter is represented by
\[
w^{(k)}=\left(
\begin{array}{cc}
w_{AA}^{(k)} & w_{AB}^{(k)} 
\\
w_{BA}^{(k)} & w_{BB}^{(k)} 
\end{array}\right). 
\]
Note that, $L_{AA}^{(2)}$, $L_{BB}^{(2)}$, $L_{BA}^{(N)}$, and $L_{BB}^{(N)}$ are
the empty matrices, because $H^*_1=H_1$ and $N^*_{N^*}=H_N$. 

We divide the layers of the model into $1 \leq k \leq N^*-1$ and $N^* \leq k \leq  N$. 
In $k = 2$, the optimal parameter $\hat{w}, \hat{b}$ are

\begin{eqnarray}
\hat{w}^{(2)} &= \left(
\begin{array}{c}
w^{*(2)}  
\\
{\cal M}^{*(2)}_{BA}
\end{array}\right)\\
\hat{b}^{(2)} &= \left(
\begin{array}{c}
b^{*(2)}  
\\
-{\cal M}^{*(2)}_{B0}
\end{array}\right),
\end{eqnarray}
where ${\cal M}^{*(2)}_{BA}$is arbitrary matrix which components are positive and $-{\cal M}^{*(2)}_{B0}$is arbitrary vector which components are negative. 

In $3 \leq k \leq N^*-1$, the optimal parameter $\hat{w}, \hat{b}$ are  
\begin{eqnarray}
\hat{w}^{(k)} &= \left(
\begin{array}{cc}
w^{*(k)}  & {\cal M}^{*(k)}_{AB}
\\
-{\cal M}^{*(k)}_{BA} & {\cal M}^{*(k)}_{BB}
\end{array}\right)\\
\hat{b}^{(k)} &= \left(
\begin{array}{c}
b^{*(k)}  
\\
-{\cal M}^{*(k)}_{B0}
\end{array}\right),
\end{eqnarray}
where ${\cal M}^{*(k)}_{AB}, {\cal M}^{*(k)}_{BB}$ are arbitrary matrices, $-{\cal M}^{*(k)}_{BA}$ are arbitrary matrices which components are negative and $-{\cal M}^{*(k)}_{B0}$ are arbitrary vector which components are negative. For $k = 3$, ${\cal M}^{*(k)}_{AB} = 0$. In each layer$(k \geq 2)$ the output $f^{(k)}(\hat{w},\hat{b},x)$ are positive. From this positivity, the following equation holds in $3 \leq k \leq N^*-1$  
\begin{align}
f_A^{(k)}(\hat{w},\hat{b},x) &= f^{(k)}(w^*,w^*,x), \\
f_B^{(k)}(\hat{w},\hat{b},x) &= 0.
\end{align}
Figure\ref{fig_op1} shows the relationships between units in this optimal parameter in $3 \leq k \leq N^*-1$.
\begin{figure}[htbp]
  \begin{minipage}[b]{5cm}
    \centering
    \raisebox{8mm}{\includegraphics[keepaspectratio, scale=0.31]{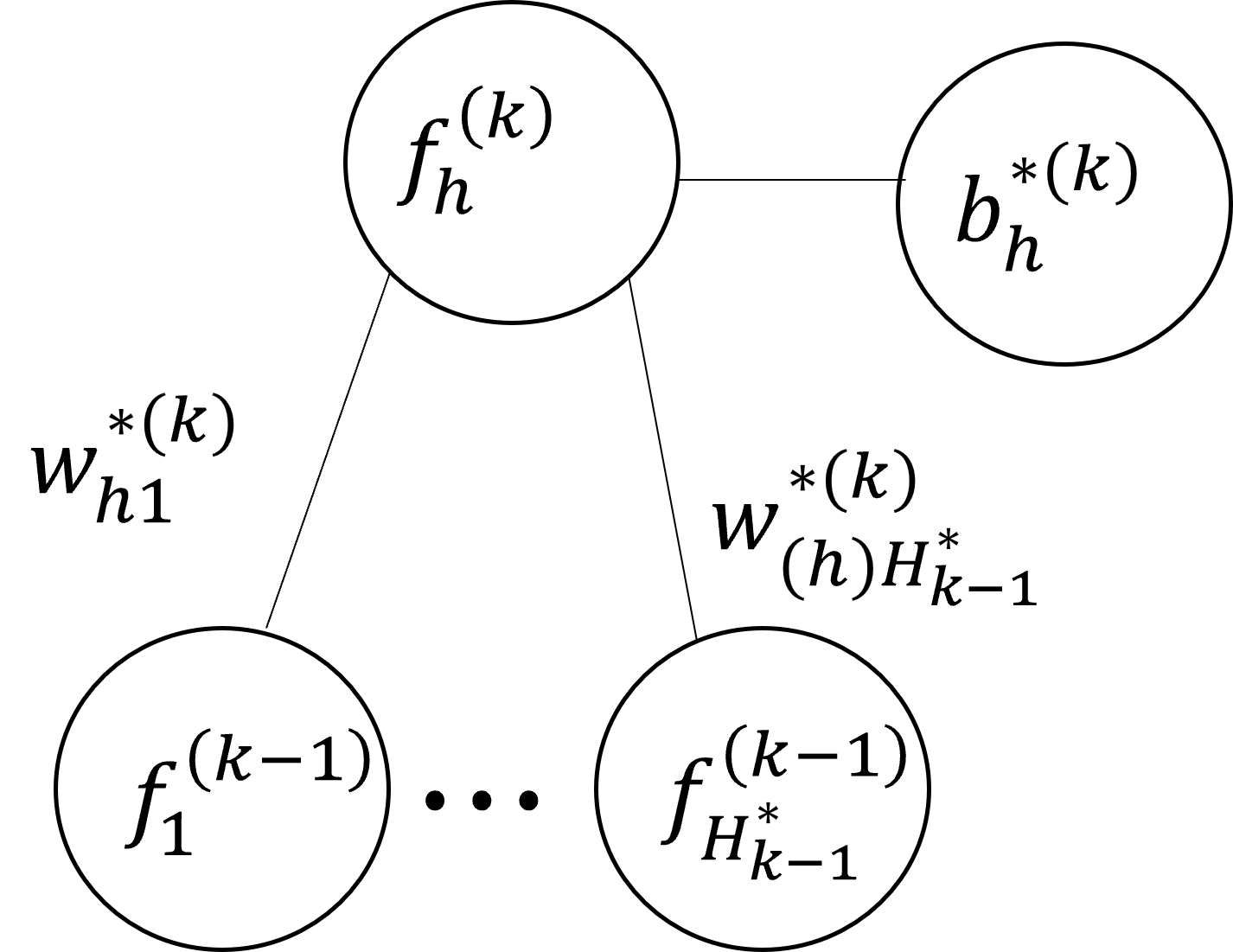}}
    \subcaption{true}
  \end{minipage}
  \begin{minipage}[b]{5cm}
    \centering
    \includegraphics[keepaspectratio, scale=0.3]{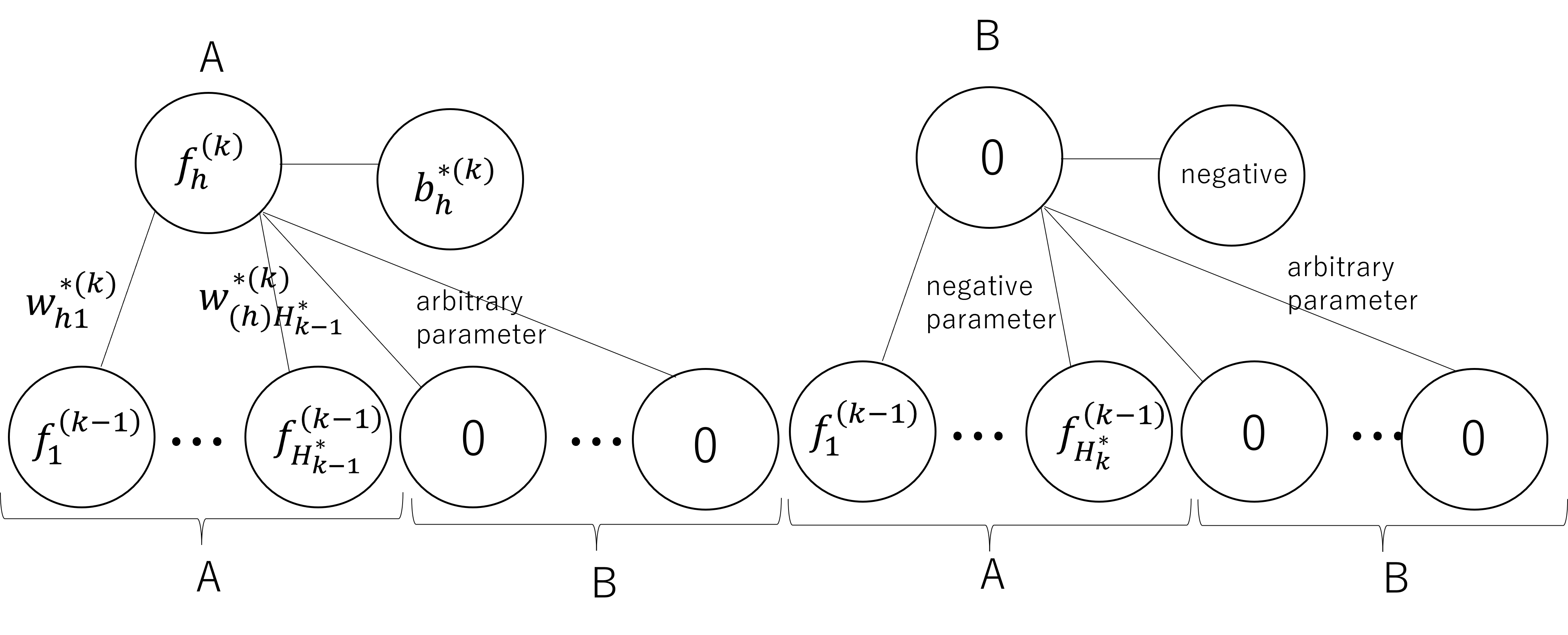}
    \subcaption{optimal}
  \end{minipage}
  \caption{The relationship between true distribution and optimal parameter in the model}
  \label{fig_op1}
\end{figure}

In $N^* \leq k \leq N - 1$, the optimal parameters $\hat{w}, \hat{b}$ are 
\begin{eqnarray}
\hat{w}^{(k)} &= \left(
\begin{array}{cc}
I_{N^*-1}  & {\cal M}^{*(k)}_{AB}
\\
-{\cal M}^{*(k)}_{BA} & -{\cal M}^{*(k)}_{BB}
\end{array}\right), \\
\hat{b}^{(k)} &= \left(
\begin{array}{c}
{\cal M}^{*(k)}_{A0} 
\\
-{\cal M}^{*(k)}_{B0}
\end{array}\right),
\end{eqnarray}
where $I_{N^*-1}$ is $H^*_{N^*-1}$ dimensional identity matrices, ${\cal M}^{*(k)}_{AB}$ are arbitrary matrices, $-{\cal M}^{*(k)}_{BA}, -{\cal M}^{*(k)}_{BB}$ are arbitrary matrices which components are negative ,${\cal M}^{*(k)}_{A0}$ are arbitrary vectors which components are positive and $-{\cal M}^{*(k)}_{B0}$ are  arbitrary vectors which components are negative. $\hat{w}^{(k)}$ satisfy that $\mathrm{Rank}(\hat{w}^{(k)})\geq H^*_{N^* - 1}$ In $k = N$, the optimal parameters $\hat{w}, \hat{b}$ are 
\begin{align}
\hat{w}^{(N)} &= \left(w^{*(N^*)}, {\cal M}^{*(N)}_{AB}\right),\\
\hat{b}^{(N)} &= b^{*(N^*)} - w^{*(N^*)}\sum_{k = N^*}^{N-1}{\cal M}^{*(k)}_{A0} ,
\end{align}
where ${\cal M}^{*(N)}_{AB}$ is an arbitrary matrix.
From the positivity of output in each layer, the following equation holds in $N^* \leq k \leq N - 1$,
\begin{align}
f_A^{(k)}(\hat{w},\hat{b},x) &= f_A^{(k - 1)}(\hat{w},\hat{b},x) + {\cal M}^{*(k)}_{A0} \\
f_B^{(k)}(\hat{w},\hat{b},x) &= 0.
\end{align}
Therefore, the following equation holds
\begin{align}
f^{(N)}(\hat{w},\hat{b},x) = f^{(N^*)}(w^*,b^*,x).
\end{align}
This equation is equivalent to 
\begin{align}
q(y|x) = p(y|\hat{w}, \hat{b},x).
\end{align}
Figure\ref{fig_nn} shows the outline of the optimal parameter introduced here.
\begin{figure}
   \centering
   \includegraphics[keepaspectratio, scale=0.3]{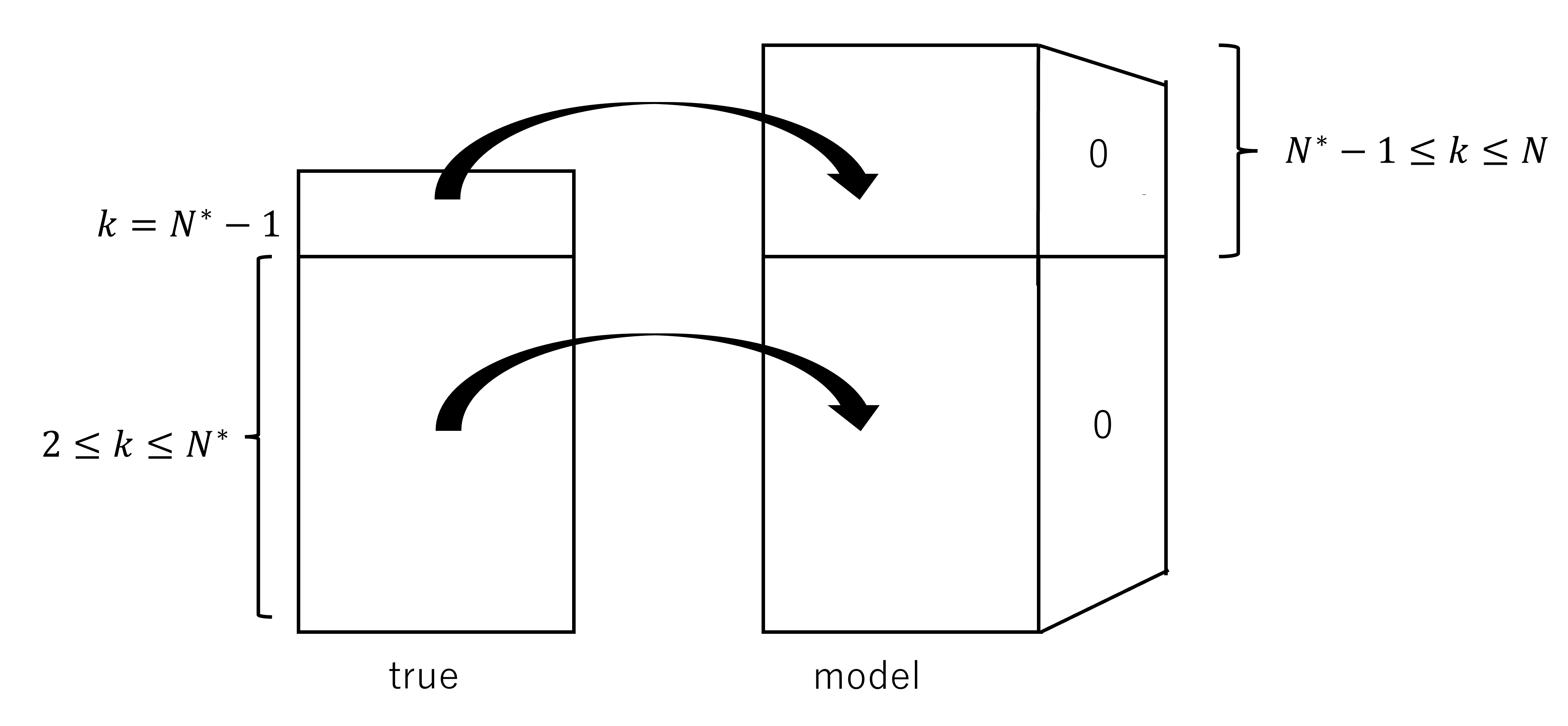}
   \centering
  \caption{Outline of optimal parameter}
  \label{fig_nn}
\end{figure}
Other than this optimal parameter, there exists various optimal parameters.

\subsection{Main Theorem}

In this subsection the main result of this paper is introduced.

\begin{theorem}(Main Theorem)\label{theorem:111}
Assume that the learning machine and the data generating distribution 
are given by $p(y|x,w,b)$ and $q(y|x)=p(y|x,w^*,b^*)$ which satisfy the conditions
eq.(\ref{eq:cond1}), eq.(\ref{eq:cond2}), and eq.(\ref{eq:cond3}), and that a sample 
$\{(X_i,Y_i)\;\;i=1,2,...,n\}$ is independently subject to $q(x)q(y|x)$. Then 
the average free energy satisfies the inequality, 
\[
\EE[F_n]\leq nS+ \lambda_{ReLU} \log n +C.
\]
For general cases, 
\begin{align}\label{eq:lambda1}
\lambda_{ReLU}=\frac{1}{2}
\left(H^*_3(H_2-H^*_2)+ \sum_{k=2}^{N^*} H^*_{k} (H_{k-1}^*+1)
\right).
\end{align}
If the support of the input distribution is bounded or contined in 
nonnegative region, 
\begin{align}\label{eq:lambda2}
\lambda_{ReLU}=\frac{1}{2}
\left(\sum_{k=2}^{N^*} H^*_{k} (H_{k-1}^*+1)
\right).
\end{align}
These results show that the average free energy is bounded, 
even if the number of layers are larger than necessary to estimate
the data-generation network. In particular eq.\eqref{eq:lambda1}
is equal to the half the number of parameters in the data-generating network.

\end{theorem}

\section{Lemmas}

In this section, we prepare several lemmas which are necessary to
prove the main theorem. 

Let the Kullback-Leibler divergence of a data-generating network 
$q(y|x)=p(y|x,w^*,b^*)$ and 
a learning machine $p(y|x)$ be 
\[
K(w,b)=\int q(x)q(y|x)\log\frac{q(y|x)}{p(y|x,w,b)}\mathrm{d}x\mathrm{d}y. 
\]
It is well-known that $K(w,b)\geq 0$ for an arbitrary $(w,b)$ and 
$K(w,b)=0$ if and only if $q(y|x)=p(y|x,w,b)$. 

\begin{lemma} \label{lemma:111}
Assume that a set $W$ is contained in the set determined by 
the prior distribution $\{(w,b);\varphi(w,b)>0\}$. Then for an 
arbitrary postive integer $n$, 
\[
\EE[F_n]\leq nS -\log \int_{W} \exp(-n K(w,b))\varphi(w,b) \mathrm{d}w \mathrm{d}b.
\]
\end{lemma}
\begin{proof}  An empirical Kullback-Leibler divergence is defined by
\[
K_n(w,b)=\frac{1}{n}\sum_{i=1}^n \log\frac{p(Y_i|X_i,w^*,b^*)}
{p(Y_i|X_i,w,b)},
\]
which satisfies $\EE[K_n](w,b)]=K(w,b)$. 
\begin{align}
\frac{q(y^n|x^n)}{p(y^n|x^n)}
& =\exp(-\sum_{i=1}^n \log \frac{q(Y_i|X_i)}{p(Y_i|X_i,w,b)})
\\
&=\exp(-nK_n(w,b)).
\end{align}
By using eq.\eqref{eq:111}, 
\begin{align}
\EE[F_n]&= -\EE[\log \frac{q(y^n|x^n)}{p(y^n|x^n)}]+n S
\\
&=-\EE[\log \int \varphi(w,b) \exp(-nK_n(w,b))dwdb] +nS. 
\end{align}
By applying Lemma.1 in \cite{Watanabe2001b}, 
\begin{align}
\EE[F_n] & \leq -\log \int \varphi(w,b) \exp(-\EE[nK_n(w,b)])\mathrm{d}w\mathrm{d}b +nS
\\
&\leq -\log  \int \varphi(w,b) \exp(-nK(w,b))\mathrm{d}w\mathrm{d}b +nS
\\
&\leq -\log \int_{W} \varphi(w,b) \exp(-nK(w,b))\mathrm{d}w\mathrm{d}b +nS,
\end{align}
where the last inequality is derived the fact that the restriction of
integrated region makes the integration not larger. 
\end{proof}

\begin{lemma}\label{lemma:222}
For  arbitrary vectors $s,t$, 
\[
\|\sigma(s)-\sigma(t)\|\leq \|s-t\|.
\]
\end{lemma}
\begin{proof}
If $s_i,t_i\geq 0$ or $s_i,t_i\leq 0$, then $|\sigma_i(s)-\sigma_i(t)|= |s_i-t_i|$. 
If $s_i\geq 0, t_i<0$, then $|\sigma_i(s)-\sigma_i(t)|=|s_i|\leq |s_i-t_i|$. 
If $s_i<0, t_i\geq 0$, then $|\sigma_i(s)-\sigma_i(t)|=|t_i|\leq |s_i-t_i|$. Hence
\[
\|\sigma(s)-\sigma(t)\|^2=\sum_{i}|\sigma_i(s)-\sigma_i(t)|^2
\leq \sum_{i}|s_i-t_i|^2=\|s-t\|^2. 
\]
\end{proof}

\begin{lemma} \label{lemma:333}
For arbitrary $w$,$w'$, $b$, $b'$, the following inequality holds, 
\begin{align}
& \|f^{(k)}(w,b,x)-
f^{(k)}(w',b',x) \|\nonumber
\\
&\leq \|w^{(k)}-w'^{(k)}\|\|f^{(k-1)}(w,b,x)\|
+\|b^{(k)}-b'^{(k)}\|\nonumber
\\
&+\|w^{(k)}\|\|f^{(k-1)}(w,b,x)-f^{(k-1)}(w',b',x)\|,
\end{align}
where $\|w^{(k)}\|$ is the operator norm of a matrix $w^{(k)}$. 
\end{lemma}
\begin{proof}
\begin{align}
& f^{(k)}(w,b,x)-
f^{(k)}(w',b',x)\nonumber 
\\
& =\sigma(w^{(k)} f^{(k-1)}(w,b,x)+b^{(k)})
-\sigma(w'^{(k)} f^{(k-1)}(w,b,x)+b'^{(k)})\nonumber 
\\
&
+\sigma(w'^{(k)} f^{(k-1)}(w,b,x)+b'^{(k)})
-\sigma(w'^{(k)} f^{(k-1)}(w',b',x)+b'^{(k)}). 
\end{align}
Hence by using Lemma \ref{lemma:222}, 
\begin{align}
& \|f^{(k)}(w,b,x)-
f^{(k)}(w',b',x) \|\nonumber 
\\
& \leq \|\sigma(w^{(k)} f^{(k-1)}(w,b,x)+b^{(k)})
-\sigma(w'^{(k)} f^{(k-1)}(w,b,x)+b'^{(k)})\|\nonumber
\\
&
+\|\sigma(w'^{(k)} f^{(k-1)}(w,b,x)+b'^{k)})
-\sigma(w'^{(k)} f^{(k-1)}(w',b',x)+b'^{(k)})\|\nonumber 
\\
&\leq \|w^{(k)}-w'^{(k)}\|\|f^{(k-1)}(w,b,x)\|
+\|b^{(k)}-b'^{(k)}\|\nonumber
\\
&+\|w'^{(k)}\|\|f^{(k-1)}(w,b,x)-f^{(k-1)}(w',b',x)\|.
\end{align}
Hence lemma is proved. 
\end{proof}

\begin{lemma} \label{lemma:444}
For arbitrary $w,b,x$, 
\begin{align}
 \|f^{(k)}(w,b,x) \|
&\leq \|w^{(k)}\|\|w^{(k-1)}\|\cdots \|w^{(2)}\|\|x\|\\
&+\|b^{(k)}\|
+ \sum_{j=1}^{k-2} \|w^{(k)}\|\|w^{(k-1)}\|\cdots \|w^{(k-j)}\|\|b^{(k-j)}\|.
\end{align}
\end{lemma}
\begin{proof} By substituting $ w':=0$ and $b'=0$, 
in Lemma \ref{lemma:333}, it follows that 
\begin{align}
 \|f^{(k)}(w,b,x) \|
\leq \|w^{(k)}\|\|f^{(k-1)}(w,b,x)\|
+\|b^{(k)}\|.
\end{align}
Then mathematical induction gives the Lemma. 
\end{proof}

In order to prove the main theorem, we need several notations. 
The convergent matrix ${\cal E}^{(k)}$ and vector ${\cal E}_0^{(k)}$ 
defined by the condition that the absolute values of 
all entries are smaller than $1/\sqrt{n}$, which is denoted by
\begin{align}
{\cal E}^{(k)}=
\left(\begin{array}{cc}
{\cal E}_{AA}^{(k)} &{\cal E}_{AB}^{(k)} 
\\
{\cal E}_{BA}^{(k)}  & {\cal E}_{BB}^{(k)} 
\end{array}\right),\;\;\;
{\cal E}_0^{(k)}=
\left(\begin{array}{c}
{\cal E}_{A0}^{(k)} 
\\
{\cal E}_{B0}^{(k)} 
\end{array}\right).
\end{align}
The positive-small-constant matrix ${\cal D}^{(k)}$  and vector ${\cal D}_0^{(k)}$ are defined 
by the condition that all entries are positive and smaller than $\delta>0$
where $\delta$ does not depend on $n$,  which is denoted by
\begin{align}
{\cal D}^{(k)}=
\left(\begin{array}{cc}
{\cal D}_{AA}^{(k)} &{\cal D}_{AB}^{(k)} 
\\
{\cal D}_{BA}^{(k)}  & {\cal D}_{BB}^{(k)} 
\end{array}\right),\;\;\;
{\cal D}_0^{(k)}=
\left(\begin{array}{c}
{\cal D}_{A0}^{(k)} 
\\
{\cal D}_{B0}^{(k)} 
\end{array}\right).
\end{align}
The positive constant matrix ${\cal M}^{(k)}$ and vector ${\cal M}_0^{(k)}$ are defined by the condition that
all entries are in the inverval $[A,B]$, 
\begin{align}
{\cal M}^{(k)}=
\left(\begin{array}{cc}
{\cal M}_{AA}^{(k)} &{\cal M}_{AB}^{(k)} 
\\
{\cal M}_{BA}^{(k)}  & {\cal M}_{BB}^{(k)} 
\end{array}\right),\;\;\;
{\cal M}_0^{(k)}=
\left(\begin{array}{c}
{\cal M}_{A0}^{(k)} 
\\
{\cal M}_{B0}^{(k)} 
\end{array}\right).
\end{align}
To prove Theorem \ref{theorem:111}, we show an upper bound of
 $\EE[F_n]$ is given 
by choosing a set $W_E$ which consists of essential weight and bias parameters. 

\noindent{\bf Definition}. (Essential parameter set $W_E$). 
A parameter $(w,b)$ is said to be in an essential parameter set $W_E$ 
if it satisfies the following conditions, (1), (2), and (3). \\
(1) For $2\leq k\leq N^*-1$, there exist 
 convergent matrices ${\cal E}^{(k)}$ and positive constant matrices $ {\cal M}^{(k)}$ 
 such that 
\begin{align}\label{eq:def11}
w^{(k)}&=
\left(\begin{array}{cc}
(w^*)^{(k)}+{\cal E}_{AA}^{(k)} & {\cal Z}_{AB}^{(k)}
\\
-{\cal M}_{BA}^{(k)}  & -{\cal M}_{BB}^{(k)} 
\end{array}\right),
\\
b^{(k)}&=
\left(\begin{array}{c}
(b^*)^{(k)}+ {\cal E}_{A0}^{(k)} 
\\ 
-{\cal M}_{B0}^{(k)} 
\end{array}\right),\label{eq:def12}
\end{align}
where
\begin{align}\label{eq:casebycase}
 {\cal Z}_{AB}^{(k)}
 =
\left\{\begin{array}{cc}
{\cal E}^{(3)}_{AB}&(k=3)
\\ 
{\cal M}_{AB}^{(k)} &(k\neq 3)
\end{array}\right..
\end{align}
Note that, for $k=2$, ${\cal Z}_{AB}^{(k)}$, ${\cal M}_{BB}^{(k)}$, and  ${\cal M}_{B0}^{(k)}$ 
are the empty matrix. 
\\
(2) For $N^*\leq k\leq N-1$, there exist 
positive-small-constant matrix ${\cal D}^{(k)}$ and positive constant matrix ${\cal M}^{(k)}$ 
\begin{align}\label{eq:wk}
w^{(k)}&=
\left(\begin{array}{cc}
I_{N^*-1}+ {\cal D}_{AA}^{(k)} &{\cal M}_{AB}^{(k)} 
\\
-{\cal M}_{BA}^{(k)}  & -{\cal M}_{BB}^{(k)} 
\end{array}\right),
\\
b^{(k)}&=
\left(\begin{array}{c}
{\cal M}_{A0}^{(k)} 
\\ 
-{\cal M}_{B0}^{(k)} 
\end{array}\right),
\end{align}
where $I_{N^*-1}$ is $H^*_{N^*-1}$ dimensional identity matrix. 
\\
(3) For $k=N$, there exist convergent matrix $ {\cal E}^{(N)}$ and vector ${\cal E}^{(N)}_{0}$ 
such that 
\begin{align}
w^{(N)}&=
\left(\begin{array}{cc}
(w^*)^{(N^*)}P^{-1}+ {\cal E}^{(N)}_{AA}, &{\cal M}^{(N)}_{AB}
\end{array}\right)
\\
b^{(N)}&=(b^*)^{(N^*)}-
\sum_{k=N^*}^{N}w^{(N)}w^{(N-1)}\cdots w^{(k)}b^{(k-1)}
+{\cal E}^{(N)}_{A0},
\end{align}
where $P\in\RR^{(H^*_{N^*-1})\times (H^*_{N^*-1})}$ is defined by
matrices in eq.\eqref{eq:wk}
\[
P=w_{AA}^{(N-1)}w_{AA}^{(N-2)}\cdots w_{AA}^{(N^*)}.
\]
Note that a positive constant $\delta>0$ is taken sufficiently small such that
arbitrary $w_{AA}^{(k)}$ $(N^*\leq k\leq N-1)$ is invertible. 

\begin{lemma} \label{lemma:555}
Assume that the weight and bias parameters are in 
the essential set $W_E$. Then 
there exist constants $c_1,c_2>0$ such that 
\begin{align}
\|f_A^{(N^*-1)}(w,b,x)-f^{(N^*-1)}(w^*,b^*,x)\|&\leq 
\frac{c_1}{\sqrt{n}}(\|x\|+1),\label{eq:lemma5551}
\\
\|f_A^{(N^*-1)}(w,b,x)\|&\leq  c_2(\|x\|+1). \label{eq:lemma5552}
\end{align}
\end{lemma}
\begin{proof} 
Eq.\eqref{eq:lemma5552} is derived from Lemma \ref{lemma:444}. 
By the definitions \eqref{eq:def11}, \eqref{eq:def12}, 
for $4\leq k\leq N^*-1$
\begin{align}
f_A^{(2)}(w,b,x)
&=\sigma(((w^*)^{(2)}+{\cal E}_{AA}^{(2)})f_A^{(1)}(w,b,x)+
(b^*)^{(2)}+{\cal E}_{A0}^{(2)}),
\\
f_A^{(3)}(w,b,x)
&=\sigma(((w^*)^{(3)}+{\cal E}_{AA}^{(3)})f_A^{(2)}(w,b,x)
\nonumber
\\
&+{\cal E}_{AB}^{(3)}f_B^{(2)}(w,b,x)+
(b^*)^{(3)}+{\cal E}_{A0}^{(3)}),\label{eq:f2(2)}
\\
f_A^{(k)}(w,b,x)
&=\sigma(((w^*)^{(k)}+{\cal E}_{AA}^{(k)})f_A^{(k-1)}(w,b,x)
\nonumber
\\
&+{\cal M}_{AB}^{(k)}f_B^{(k-1)}(w,b,x)+
(b^*)^{(k)}+{\cal E}_{A0}^{(k)}). 
\end{align}
Here, for $4\leq k\leq N^*-1$, $f_B^{(k-1)}(w,b,x)=0$, since all entries of 
$w_{BA}^{(k-1)}$, $w_{BB}^{(k-1)}$, and $w_{B0}^{(k-1)}$ are negative and 
the output of ReLU function $f^{(k-2)}(w,b,x)$ is nonnegative. On the other hand, 
\begin{align}
f^{(k)}(w^*,b^*,x)
&=\sigma((w^*)^{(k)}f^{(k-1)}(w^*,b^*,x)+(b^*)^{(k)}).
\end{align}
Hence by Lemma \ref{lemma:333}, $2\leq k\leq N^*-1$,
\begin{align}
& \|f_A^{(k)}(w,b,x)-f^{(k)}(w^*,b^*,x)\|
\\
& \leq \|{\cal E}_{AA}^{(k-1)}f_A^{(k-1)}(w,b,x)+{\cal E}_{A0}^{(k)}\|
+\delta_{k,3}\|{\cal E}_{AB}^{(3)}f_B^{(2)}(w,b,x)\|
\\
&+\|(w^*)^{(k)}(f_A^{(k-1)}(w,b,x)-f^{(k-1)}(w^*,b^*,x))\|
\\
& \leq \|{\cal E}_{AA}^{(k-1)}\|\|f_A^{(k-1)}(w,b,x)\|+\|{\cal E}_{A0}^{(k)}\|
+\delta_{k,3}\|{\cal E}_{AB}^{(3)}\|\|f_B^{(2)}(w,b,x)\|
\\
&+\|(w^*)^{(k)}\|\|f_A^{(k-1)}(w,b,x)-f^{(k-1)}(w^*,b^*,x)\|,
\end{align}
where $\delta_{k,3}=1$ if $k=1$ or 0 otherwise. 
The entries of matrices in ${\cal E}_{AA}^{(k-1)}$,
 ${\cal E}_{AB}^{(3)}$, and ${\cal E}_{A0}^{(k)}$ 
are bounded by $1/\sqrt{n}$ order term and 
the operator norm is bounded by the Frobenius norm, hence
$\|{\cal E}_{AA}^{(k-1)}\|$,
 $\|{\cal E}_{AB}^{(3)}\|$, and $\|{\cal E}_{A0}^{(k)}\|$ are bounded by 
$1/\sqrt{n}$ order term. Moreover $\|(w^*)^{(k)}\|$ is a constant term. 
For $k=2$, $f_A^{(k-1)}(w,b,x)-f^{(k-1)}(w^*,b^*,x)=x-x=0$. Then by using mathematical
induction we obtain the Lemma.
\end{proof}

\begin{lemma}\label{lemma:666}
Assume that the weight and bias parameters are in 
the set $W_E$. Then there exists a constant $c_3>0$ such that 
\begin{align}
\|f^{(N)}(w,b,x)-f^{(N^*)}(w^*,b^*,x)\|&\leq \frac{c_3}{\sqrt{n}}(\|x\|+1).
\end{align}
\end{lemma}
\begin{proof}
Let $h\in\RR^{H_{N}}$ and $h^*\in\RR^{H^*_{N^*}}$ 
($H_N=H^*_{N^*}$) be input vectors 
into the output layers of the learning and data-generating machines
respectively. In other words, $h$ and $h^*$ is defined such that 
$f^{(N)}(w,b,x)=\sigma(h)$ and
$f^{(N^*)}(w^*,b^*,x)=\sigma(h^*)$. 
By the definition of the essential parameter set (2), 
for $N^*-1\leq k \leq N-1$, all entries of 
$w_{BA}^{(k)}$, $w_{BB}^{(k)}$ and $b_{B0}^{(k)}$ 
are negative. Hence, for $N^*\leq k \leq N-1$, $f_2^{(k)}(w,b,x)=0$. 
For $N^*\leq k \leq N-1$, all entries of 
$w_{AA}^{(k)}$, $w_{AB}^{(k)}$ and $b_{A0}^{(k)}$ 
are positive. Hence
by using $\sigma(t)=t$ for $t\geq 0$, 
\begin{align}
h&=w_{AA}^{(N)}w_{AA}^{(N-1)}\cdots w_{AA}^{(N^*)} f_A^{(N^*-1)} (w,b,x)
\\&
+b^{(N)}+\sum_{k=N^*}^{N}w_{AA}^{(N)}\cdots w_{AA}^{(k)}b_{A0}^{(k-1)}.
\end{align}
On the other hand,
\begin{align}
h^*=(w^*)^{(N^*)} f^{(N^*-1)} (w^*,b^*,x)+(b^*)^{(N)}.
\end{align}
If $w$ is in the essential set of parameters, 
\begin{align}
& w_{AA}^{(N)}w_{AA}^{(N-1)}\cdots w_{AA}^{(N^*)} 
 = ((w^*)^{(N^*)}P^{-1}+ {\cal E}^{(N)}_{AA})w_{AA}^{(N-1)}\cdots w_{AA}^{(N^*)} 
\\
&= (w^*)^{(N^*)}+ {\cal E}^{(N)}_{AA}w_{AA}^{(N-1)}\cdots w_{AA}^{(N^*)} .
\end{align}
It follows that
\begin{align}
&\|w^{(N)}w^{(N-1)}\cdots w^{(N^*)} f^{(N^*-1)} (w,b,x)
-w^{(N*)} f^{(N^*-1)} (w^*,b^*,x)\|  
\\
&\leq 
\|w_{AA}^{(N)}w_{AA}^{(N-1)}\cdots w_{AA}^{(N^*)} f_A^{(N^*-1)} (w,b,x)
-w^{(N*)} f^{(N^*-1)} (w^*,b^*,x)\|  
\\
&
\leq \|
(w^*)^{(N^*)}(f_A^{(N^*-1)} (w,b,x)
- f^{(N^*-1)} (w^*,b^*,x))
\|
\\
&+
\| {\cal E}_{AB}^{(N)}\|\|w_{AA}^{(N-1)}\|\cdots 
\|w_{AA}^{(N^*)}\|\|f_A^{(N^*-1)} (w^*,b^*,x)\|  
\\
& \leq \frac{c_4}{\sqrt{n}}(\|x\|+1),
\end{align}
where the last inequality is derived by Lemma \ref{lemma:555}.
Also by the definition, 
\begin{align}
\|b^{(N)}+\sum_{k=N^*}^{N}w^{(N)}\cdots w^{(k)}b^{(k-1)}
-(b^*)^{(N^*)}  \|\leq \frac{c_4}{\sqrt{n}},
\end{align}
it follows that
\[
\|h-h^*\|\leq \frac{c_5}{\sqrt{n}}(\|x\|+1).
\]
Then applying Lemma \ref{lemma:222} completes the lemma. 
\end{proof}

\begin{lemma}\label{lemmma:777}
(1) If the support of $q(x)$ is contained in a positive region, 
the same conclusion as Lemma \ref{lemma:555} holds by 
replacing ${\cal Z}_{AB}^{(3)}$ in \eqref{eq:casebycase}
with ${\cal M}_{AB}^{(3)}$.  \\
(2) 
If the support of $q(x)$ is contained in a bounded region, 
the same conclusion as Lemma \ref{lemma:555} holds by 
replacing ${\cal Z}_{AB}^{(3)}$ in \eqref{eq:casebycase}
with ${\cal M}_{AB}^{(3)}$ and by replacing 
$-{\cal M}_{B0}^{(3)}$ in \eqref{eq:def12} 
 with a matrix in a sufficiently small region.
\end{lemma}
\begin{proof}
In both cases, 
$f_B^{(2)}(w,b,x)=0$ in eq.\eqref{eq:f2(2)} holds. hence the same 
conclusion of Lemma \ref{lemma:555} holds. 
\end{proof}
\vskip3mm

\section{Proof of Main Theorem}

In this section we prove the main theorem.

\begin{proof} (Main theorem). 
By Lemma \ref{lemma:111}, it is sufficient to prove
that there exists a constant $C>0$ such that 
\[
\int_{W_E}\exp(-nK(w,b))\varphi(w,b)\mathrm{d}w\mathrm{d}b 
\geq \frac{C}{n^{\lambda}}
\]
where
\[
K(w,b)=\frac{1}{2}\int \|f^{(N)}(w,b,x)-f^{(N^*)}(w^*,b^*,x)\|^2 q(x) \mathrm{d}x. 
\]
By using Lemma \ref{lemma:666}, if $(w,b)\in W_E$, 
\[
K(w,b)\leq \frac{c_3^2}{2n}\int (\|x\|+1)^2  q(x) \mathrm{d}x =\frac{c_4}{n}<\infty.
\]
It follows that 
\begin{align}
&\int_{W_E}\exp(-nK(w,b))\varphi(w,b)\mathrm{d}w\mathrm{d}b 
\\
& \geq \exp(-c_4) \left(\min_{(w,b)\in W_E} \varphi(w,b)\right) \mbox{Vol}(W_E). 
\end{align}
where $c_4>0$, 
$\min_{(w,b)\in W_E}\varphi(w,b) >0$, and $\mbox{Vol}(W_E)$ is the volume 
of the set $W_E$ by the Lebesgue measure. 
By the definition of the essential parameter set $W_E$, its volume is 
determined by the dimension of the convergent matrices and vectors.
Let $2\lambda$ be the number of parameters in convergent matrices and vectors. 
Then 
\[
\mbox{Vol}(W_E)\geq \frac{C_1}{n^{\lambda}},
\]
where in general cases, 
\[
\lambda=
\frac{1}{2}
\left(
+ H^*_3(H_2-H^*_2)+ \sum_{k=2}^{N^*} H^*_{k} (H_{k-1}^*+1)
\right).
\]
If the support of the input distribution is contained in a positive region or
a bounded region, 
\[
\lambda=
\frac{1}{2}
\left(\sum_{k=2}^{N^*} H^*_{k} (H_{k-1}^*+1)
\right),
\]
which completes the main theorem. 
\end{proof}

\section{Discussion}

In this section, we discuss the three points in this paper.

\subsection{Property of Free Energy}

Firstly we study a monotone property of the free energy as a function
of the integrated region. 
As we have shown in the proof, the average free energy satisfies
\[ 
\EE[F_n]\leq -\log \int \exp(-nK(w,b)) \varphi(w,b) dwdb  +nS. 
\]
We define a function $G(U)$ of an integrated region $U$,
\[
G(U)= - \log \int_{U}\exp(-nK(w,b))\varphi(w,b) dw db,
\]
where $U$ is a measurable subset in $\{(w,b)\}$. Then an inequality holds, 
\[
U_1\supset U_2\Longrightarrow G(U_1)\leq G(U_2). 
\]
Hence,  an set $U$ satisfies a condition that, for $(w,b)\in U$, $K(w,b)\leq 1/n$,
it follows that 
\[
\EE[F_n]\leq G(U)+ nS, 
\]
and 
\[
G(U)\leq -\log \mbox{Vol}(U)+nS + const.,
\]
where $\mbox{Vol}(U)$ is the volume of the set $U$. 
In this paper, we chose the essential parameter set $W_E$ for 
such a subset, and showed the volume of this set is determined 
the number of the  convergent parameters.

\subsection{Sepcial Property of ReLU Function} 

Secondly, a special property of the ReLU function is discussed. 
The output of the ReLU function is nonnegative and 
equal to zero for a negative input. Hence, if all parameters and biases 
from the $(k-1)$th layer to $k$th layer are negative, then 
the output of $k$th layer is equal to zero. This property is used
to evaluate the effect by the redundant parameters in each layer.

Moreover, if all parameters and biases 
from the $(k-1)$th layer to $k$th layer are positive, 
then the output of the 
$k$th layer is a linear function of the $(k-1)$th output. 
This property is used to evaluate the effect by the redundant layers. 

These two points were employed in the mathematical proof of the 
main theorem, which might also be useful to design design a deep ReLU
neural network for the smaller generalization error. 

\subsection{Unrealizable Cases}

Thirdly we discuss a case when the data-generating distribution 
is not realizable by a ReLU function. In this paper, we assumed in the theorem 
that a sample is subject to the data-generating function 
represented by a parameter $(w^*,b^*)$. 
In real world problems, such an assumption is not
satisfied in general. 
However, if a learning machine is
sufficiently large such that there exists $(w^*,b^*)$ such that
the Kullback-Leibler distance satisfies 
\[
K(q(y|x)||p(y|x,w^*,b^*))=o(1/n),
\]
then the same inequality as the main theorem holds. 
In other words. a deep ReLU neural network in an overparametrized 
state makes the generalization error caused by bias 
small enough with the bounded error caused by variance. 
This is a good property of a deep ReLU neural network, when 
Bayesian inference is employed in learning.

\section{Conclusion}

In this paper, we studied a deep ReLU neural network in an overparametrized case, 
and derived the upper bound of Bayesian free energy. Since the generalization error
is equal to the increase of the free energy, the result of this paper shows that
the generalization error of the deep ReLU neural network is bounded even if
the number of layers are larger then necessary to approximate 
the data-generating distribution. 

\section*{Acknowledgement}

This work was partially supported by JSPS KAKENHI
Grant-in-Aid for Scientific Research (C) 21K12025.

\bibliography{ref}

\begin{thebibliography}{10}
\expandafter\ifx\csname url\endcsname\relax
  \def\url#1{\texttt{#1}}\fi
\expandafter\ifx\csname urlprefix\endcsname\relax\def\urlprefix{URL }\fi
\expandafter\ifx\csname href\endcsname\relax
  \def\href#1#2{#2} \def\path#1{#1}\fi

\bibitem{Watanabe2001b}
S.~Watanabe, Algebraic geometrical methods for hierarchical learning machines,
  Neural Networks 14~(8) (2001) 1049--1060.
\newblock \href {https://doi.org/10.1016/s0893-6080(01)00069-7}
  {\path{doi:10.1016/s0893-6080(01)00069-7}}.

\bibitem{Watanabe2007}
S.~Watanabe, Almost all learning machines are singular, in: 2007 IEEE Symposium
  on Foundations of Computational Intelligence, IEEE, 2007, pp. 383--388.

\bibitem{Amari1993}
S.~Amari, N.~Murata, Statistical theory of learning curves under entropic loss
  criterion, Neural Computation 5~(1) (1993) 140--153.
\newblock \href {https://doi.org/10.1162/neco.1993.5.1.140}
  {\path{doi:10.1162/neco.1993.5.1.140}}.

\bibitem{Akaike1974}
H.~Akaike, A new look at the statistical model identification, IEEE
  transactions on automatic control 19~(6) (1974) 716--723.
\newblock \href {https://doi.org/10.1109/tac.1974.1100705}
  {\path{doi:10.1109/tac.1974.1100705}}.

\bibitem{Schwarz1978}
G.~Schwarz, Estimating the dimension of a model, The annals of statistics 6~(2)
  (1978) 461--464.
\newblock \href {https://doi.org/10.1214/aos/1176344136}
  {\path{doi:10.1214/aos/1176344136}}.

\bibitem{Watanabe2001a}
S.~Watanabe, Algebraic analysis for nonidentifiable learning machines, Neural
  Computation 13~(4) (2001) 899--933.
\newblock \href {https://doi.org/10.1162/089976601300014402}
  {\path{doi:10.1162/089976601300014402}}.

\bibitem{Aoyagi2012}
M.~Aoyagi, K.~Nagata, Learning coefficient of generalization error in bayesian
  estimation and vandermonde matrix-type singularity, Neural Computation 24~(6)
  (2012) 1569--1610.
\newblock \href {https://doi.org/10.1162/neco_a_00271}
  {\path{doi:10.1162/neco_a_00271}}.

\bibitem{Hartigan1985}
J.~A. Hartigan, A failure of likelihood asymptotics for normal mixtures, in:
  Proceedings of the Barkeley Conference in Honor of Jerzy Neyman and Jack
  Kiefer, 1985, Vol.~2, 1985, pp. 807--810.

\bibitem{Nakajima2010}
S.~M. Nakajima, S., Implicit regularization in variational bayesian matrix
  factorization., in: ICML, 2010, pp. 815--822.
\newblock \href {https://doi.org/10.5555/3104322.3104426}
  {\path{doi:10.5555/3104322.3104426}}.

\bibitem{Aoyagi2005}
M.~Aoyagi, S.~Watanabe, Stochastic complexities of reduced rank regression in
  bayesian estimation, Neural Networks 18~(7) (2005) 924--933.
\newblock \href {https://doi.org/10.1016/j.neunet.2005.03.014}
  {\path{doi:10.1016/j.neunet.2005.03.014}}.

\bibitem{Sato2019}
K.~Sato, S.~Watanabe, Bayesian generalization error of poisson mixture and
  simplex vandermonde matrix type singularity, arXiv preprint arXiv:1912.13289
  (2019).

\bibitem{Hayashi2021}
N.~Hayashi, The exact asymptotic form of bayesian generalization error in
  latent dirichlet allocation, Neural Networks 137 (2021) 127--137.
\newblock \href {https://doi.org/10.1016/j.neunet.2021.01.024}
  {\path{doi:10.1016/j.neunet.2021.01.024}}.

\bibitem{Yamazaki2005}
K.~Yamazaki, S.~Watanabe, Singularities in complete bipartite graph-type
  boltzmann machines and upper bounds of stochastic complexities, {IEEE}
  Transactions on Neural Networks 16~(2) (2005) 312--324.
\newblock \href {https://doi.org/10.1109/tnn.2004.841792}
  {\path{doi:10.1109/tnn.2004.841792}}.

\bibitem{Aoyagi2010}
M.~Aoyagi, A bayesian learning coefficient of generalization error and
  vandermonde matrix-type singularities, Communications in Statistics - Theory
  and Methods 39~(15) (2010) 2667--2687.
\newblock \href {https://doi.org/10.1080/03610920903094899}
  {\path{doi:10.1080/03610920903094899}}.

\bibitem{Nagayasu2022}
S.~Nagayasu, S.~Watanbe, Asymptotic behavior of free energy when optimal
  probability distribution is not unique, Neurocomputing 500 (2022) 528--536.
\newblock \href {https://doi.org/10.1016/j.neucom.2022.05.071}
  {\path{doi:10.1016/j.neucom.2022.05.071}}.

\bibitem{Wei2022}
S.~Wei, D.~Murfet, M.~Gong, H.~Li, J.~Gell-Redman, T.~Quella, Deep learning is
  singular, and that's good, {IEEE} Transactions on Neural Networks and
  Learning Systems (2022) 1--14\href
  {https://doi.org/10.1109/tnnls.2022.3167409}
  {\path{doi:10.1109/tnnls.2022.3167409}}.

\bibitem{Tokuda2022}
S.~Tokuda, K.~Nagata, M.~Okada, Intrinsic regularization effect in bayesian
  nonlinear regression scaled by observed data, Physical Review Research 4~(4)
  (dec 2022).
\newblock \href {https://doi.org/10.1103/physrevresearch.4.043165}
  {\path{doi:10.1103/physrevresearch.4.043165}}.

\bibitem{Drton2017}
M.~Drton, M.~Plummer, A bayesian information criterion for singular models,
  Journal of the Royal Statistical Society: Series B (Statistical Methodology)
  79~(2) (2017) 323--380.
\newblock \href {https://doi.org/10.1111/rssb.12187}
  {\path{doi:10.1111/rssb.12187}}.

\bibitem{Yamazaki2013}
K.~Yamazaki, D.~Kaji, Comparing two bayes methods based on the free energy
  functions in bernoulli mixtures, Neural Networks 44 (2013) 36--43.
\newblock \href {https://doi.org/10.1016/j.neunet.2013.03.002}
  {\path{doi:10.1016/j.neunet.2013.03.002}}.

\bibitem{Nagata2008}
K.~Nagata, S.~Watanabe, Asymptotic behavior of exchange ratio in exchange monte
  carlo method, Neural Networks 21~(7) (2008) 980--988.
\newblock \href {https://doi.org/10.1016/j.neunet.2007.11.002}
  {\path{doi:10.1016/j.neunet.2007.11.002}}.

\bibitem{Watanabe2009}
S.~Watanabe, Algebraic Geometry and Statistical Learning Theory, Cambridge
  University Press, 2009.
\newblock \href {https://doi.org/10.1017/cbo9780511800474}
  {\path{doi:10.1017/cbo9780511800474}}.

\bibitem{Rissanen_1978}
J.~Rissanen, Modeling by shortest data description, Automatica 14~(5) (1978)
  465--471.
\newblock \href {https://doi.org/10.1016/0005-1098(78)90005-5}
  {\path{doi:10.1016/0005-1098(78)90005-5}}.

\bibitem{S_Watanabe2013}
S.~Watanabe, A widely applicable bayesian information criterion, The Journal of
  Machine Learning Research 14~(1) (2013) 867--897.
\newblock \href {https://doi.org/10.5555/2567709.2502609}
  {\path{doi:10.5555/2567709.2502609}}.

\bibitem{S_Watanabe2010a}
S.~Watanabe, Asymptotic equivalence of bayes cross validation and widely
  applicable information criterion in singular learning theory., Journal of
  machine learning research 11~(12) (2010).
\newblock \href {https://doi.org/10.5555/1756006.1953045}
  {\path{doi:10.5555/1756006.1953045}}.

\bibitem{Hironaka1964}
H.~Hironaka, Resolution of singularities of an algebraic variety over a field
  of characteristic zero: Ii, Annals of Mathematics (1964) 205--326.

\bibitem{Atiyah1970}
M.~F. Atiyah, Resolution of singularities and division of distributions,
  Communications on pure and applied mathematics 23~(2) (1970) 145--150.

\end{thebibliography}

\bibliographystyle{plain}

\end{document}